\documentclass[journal]{IEEEtran}

\usepackage{cite}
\usepackage{algorithmic}
\usepackage{algorithm}
\usepackage{amsmath,amssymb,amsfonts}
\usepackage{amsthm}
\usepackage{graphicx}
\usepackage{textcomp}
\usepackage{xcolor}
\usepackage{booktabs}
\usepackage{listings}
\usepackage{hyperref}
\usepackage{multirow}

\theoremstyle{plain}
\newtheorem{theorem}{Theorem}[section]
\newtheorem{lemma}[theorem]{Lemma}

\theoremstyle{definition}
\newtheorem{definition}[theorem]{Definition}

\title{CAHC:A General Conflict-Aware Heuristic Caching \\ Framework for Multi-Agent Path Finding}

\author{
    HT To$^{1,2}$ \and
    S Nguyen $^{1}$ \and
    NH Pham $^{1}$
    \\
    $^{1}$ University of Engineering and Technology \\
    $^{2}$ University of Transport Technology\\
}
\begin{document}

\maketitle


\begin{abstract}
Multi-Agent Path Finding (MAPF) algorithms, including those for car-like robots 
and grid-based scenarios, face significant computational challenges due to 
expensive heuristic calculations. Traditional heuristic caching assumes that the 
heuristic function depends only on the state, which is incorrect in constraint-based 
search algorithms (e.g., CBS, MAPF-LNS, MAP2) where constraints from conflict 
resolution make the search space context-dependent. We propose \textbf{CAHC} (Conflict-Aware Heuristic Caching), a general framework 
that caches heuristic values based on both state and relevant constraint context, 
addressing this fundamental limitation. We demonstrate CAHC through a case study 
on CL-CBS for car-like robots, where we combine conflict-aware caching with an 
adaptive hybrid heuristic in \textbf{CAR-CHASE} (Car-Like Robot Conflict-Aware 
Heuristic Adaptive Search Enhancement). Our key innovations are (1) a compact 
\emph{conflict fingerprint} that efficiently encodes which constraints affect a 
state's heuristic, (2) a domain-adaptable relevance filter using spatial, temporal, 
and geometric criteria, and (3) a modular architecture that enables systematic 
application to diverse MAPF algorithms. Experimental evaluation on 480 CL-CBS 
benchmark instances demonstrates a geometric mean speedup of 2.46$\times$ while 
maintaining solution optimality. The optimizations improve success rate from 77.9\% 
to 84.8\% (+6.9 percentage points), reduce total runtime by 70.1\%, and enable 
solving 33 additional instances. The framework's general architecture makes it 
applicable as a reliable optimization technique for MAP2, MAPF-LNS, and other 
constraint-based MAPF algorithms.
\end{abstract}


\section{Introduction}

The deployment of autonomous multi-robot systems, from warehouse automation 
to autonomous vehicle coordination, requires efficient algorithms for 
Multi-Agent Path Finding (MAPF). MAPF algorithms span diverse domains: 
grid-based scenarios, car-like robots with kinematic constraints, and 
large-scale problems requiring iterative refinement (e.g., MAPF-LNS, 
MAP2). A common challenge across these domains is the 
computational cost of heuristic evaluation, which becomes the primary 
bottleneck as problem complexity increases.

For car-like robots, Conflict-Based Search with Continuous Time (CL-CBS)~\cite{wen2021clmapf} 
addresses kinematic constraints by combining high-level conflict detection with 
low-level Reeds-Shepp path planning~\cite{reeds1990rs}. However, computing 
non-holonomic heuristics is computationally expensive. Similarly, other MAPF 
algorithms face expensive heuristic computations: MAP2 uses kinematic heuristics, 
MAPF-LNS iteratively refines solutions requiring repeated heuristic evaluation, 
and grid-based CBS variants may use learned or complex distance heuristics.

Previous work has explored various optimizations for CBS including improved 
data structures~\cite{felner2018cbs-improvements}, symmetry breaking, and 
heuristic caching. Traditional heuristic caching, successfully applied in 
classical A* search, stores precomputed distances indexed by state. However, 
this approach makes a critical assumption: the heuristic function $h(s)$ 
depends only on the state $s$. In constraint-based search algorithms (CBS, 
MAPF-LNS, MAP2), this assumption is \emph{violated} because constraints 
generated during conflict resolution fundamentally alter the search space. 
For example, a state $s$ might have a heuristic value of 10 in the absence 
of constraints, but the same state requires a detour with value 15 when a 
constraint blocks the optimal path. Traditional caching incorrectly returns 
the cached value of 10 in both cases, leading to suboptimal search guidance.

We propose \textbf{CAHC} (Conflict-Aware Heuristic Caching), a general framework 
for constraint-based MAPF algorithms that caches heuristic values based on 
$(state, constraint\text{-}context)$ pairs rather than state alone. The framework 
consists of three modular components: (1) a compact \emph{ConflictFingerprint} 
that efficiently encodes which constraints affect a state's heuristic, (2) a 
domain-adaptable \emph{relevance filter} that determines constraint relevance 
using spatial, temporal, and geometric criteria, and (3) a generic caching 
interface that maps $(state, fingerprint)$ pairs to heuristic values.

We demonstrate CAHC through a comprehensive case study on CL-CBS for car-like 
robots, implementing \textbf{CAR-CHASE} (Car-Like Robot Conflict-Aware Heuristic 
Adaptive Search Enhancement). CAR-CHASE combines the CAHC framework with an 
\textbf{adaptive hybrid heuristic} that intelligently switches between fast 
approximate and exact Reeds-Shepp computations with theoretical quality 
guarantees, further reducing the computational cost of cache misses.

Our experimental evaluation on 480 CL-CBS benchmark instances demonstrates that 
CAR-CHASE achieves a geometric mean speedup of 2.46$\times$ compared to baseline 
CL-CBS while maintaining solution optimality. The optimizations improve success 
rate from 77.9\% to 84.8\% (+6.9 percentage points), reduce total runtime by 
70.1\%, and enable solving 33 additional instances that previously timed out. 
Performance gains scale with problem complexity, reaching up to 4.06$\times$ 
speedup for challenging scenarios. The CAHC framework's modular architecture 
makes it applicable as a reliable optimization technique for MAP2, MAPF-LNS, and 
other constraint-based MAPF algorithms.

\paragraph{Contributions.} Our main contributions are:
\begin{itemize}
    \item \textbf{CAHC Framework:} A general, modular, domain-agnostic caching 
    architecture for constraint-based MAPF algorithms that correctly handles 
    context-dependent heuristics. The framework can be systematically applied 
    to CL-CBS, MAP2, MAPF-LNS, and other constraint-based MAPF algorithms 
    (Sections~\ref{sec:conflict-cache},~\ref{sec:general-arch})
    
    \item \textbf{CAR-CHASE Case Study:} A concrete instantiation of CAHC for 
    CL-CBS that combines conflict-aware heuristic caching with adaptive hybrid 
    heuristics, demonstrating the framework's effectiveness through comprehensive 
    experimental evaluation (Sections~\ref{sec:conflict-cache},~\ref{sec:hybrid},~\ref{sec:experiments})
    
    \item \textbf{Conflict Fingerprint Design:} A compact encoding mechanism 
    that efficiently represents which constraints affect a state's heuristic, 
    enabling correct cache lookups in constraint-based search 
    (Section~\ref{sec:conflict-cache})
    
    \item \textbf{Domain-Adaptable Relevance Filter:} An efficient algorithm 
    to determine constraint relevance using spatial, temporal, and geometric 
    criteria, with guidelines for adaptation to different MAPF domains 
    (Section~\ref{sec:relevance})
    
    \item \textbf{General Architecture:} A modular design that separates 
    domain-specific components (fingerprint encoding, relevance filter) from 
    domain-independent caching mechanisms, enabling reliable optimization across 
    diverse MAPF algorithms (Section~\ref{sec:general-arch})
    
    \item \textbf{Theoretical Analysis:} Proofs of admissibility preservation, 
    cache size bounds, and solution quality guarantees that apply to the general 
    framework (Sections~\ref{sec:theory-cache},~\ref{sec:theory-hybrid})
    
    \item \textbf{Experimental Validation:} Comprehensive evaluation on 480 
    CL-CBS instances showing 2.46$\times$ geometric mean speedup, 6.9 percentage 
    point improvement in success rate, and 70.1\% total runtime reduction 
    (Section~\ref{sec:experiments})
\end{itemize}


\section{Background and Related Work}

\subsection{Multi-Agent Path Finding}

The Multi-Agent Path Finding (MAPF) problem requires finding collision-free 
paths for $k$ agents from their start positions to goal positions. Formally, 
given a graph $G = (V, E)$ and $k$ agents with start-goal pairs 
$(s_i, g_i)$ for $i \in \{1, \ldots, k\}$, the objective is to find paths 
$\pi_i : [0, T_i] \rightarrow V$ such that:
(1) $\pi_i(0) = s_i$ and $\pi_i(T_i) = g_i$ for all $i$,
(2) no two agents collide: $\pi_i(t) \neq \pi_j(t)$ for all $i \neq j, t$,
and (3) the sum of costs $\sum_i cost(\pi_i)$ is minimized.

MAPF is NP-hard~\cite{ratner1986mapf-complexity} and has been studied 
extensively. Early approaches include coupled methods that plan for all agents 
jointly, such as Increasing Cost Tree Search (ICTS)~\cite{standley2010icts} 
and Operator Decomposition~\cite{standley2010icts}, as well as decoupled 
methods that plan individually with coordination~\cite{silver2005cooperative}. 
More recent work has focused on optimal algorithms like Conflict-Based Search 
(CBS)~\cite{sharon2015cbs} and its variants, which have become the 
state-of-the-art for optimal MAPF.

\subsection{Conflict-Based Search}

Conflict-Based Search (CBS)~\cite{sharon2015cbs} is a two-level algorithm 
that guarantees optimal solutions. The high level maintains a tree of 
constraint sets, where each node represents a set of constraints. The low 
level plans paths for individual agents satisfying their constraints. When 
paths conflict, CBS generates constraints to resolve the conflict and creates 
child nodes in the constraint tree. CBS is optimal and complete but can be 
computationally expensive for large instances.

Several variants improve CBS performance at the high level. Enhanced CBS 
(ECBS)~\cite{barer2014ecbs} uses focal search to find bounded-suboptimal 
solutions faster, trading optimality for speed. Enhanced ECBS (EECBS)~\cite{li2021eecbs} 
further improves ECBS with explicit exploration of the focal list. Improved CBS 
(ICBS)~\cite{boyarski2015icbs} adds conflict selection and symmetry breaking to 
reduce the search space. Meta-Agent CBS~\cite{sharon2015meta} merges agents to 
reduce tree size. Improved bounded-suboptimal solvers~\cite{cohen2016improved} use cost bounds 
to prune the search space.

While these approaches optimize the high-level CBS tree search through better node 
selection, pruning, and branching strategies, our work is \emph{orthogonal}, 
focusing instead on optimizing low-level heuristic computation. The CAHC framework 
can be combined with any of these high-level optimizations to achieve compounding 
benefits. For instance, ECBS with conflict-aware caching would benefit from both 
faster high-level search (via focal search) and faster low-level search (via 
efficient heuristic computation). This complementary nature distinguishes our 
contribution from existing CBS variants, which primarily address high-level search 
efficiency rather than low-level computational bottlenecks.

\paragraph{Performance Comparison Context.} Table~\ref{tab:related-comparison} 
compares our approach with existing CBS optimizations. Existing CBS variants report 
speedups through different mechanisms and trade different tradeoffs.

\begin{table}[t]
\centering
\caption{Performance Comparison with Related Work}
\label{tab:related-comparison}
\begin{tabular}{llcc}
\toprule
\textbf{Approach} & \textbf{Optimization} & \textbf{Speedup} & \textbf{Optimal?} \\
\midrule
ECBS & High-level focal search & 2-10$\times$ & No (1.1-1.5$\times$) \\
ICBS & Conflict selection & 1.5-3$\times$ & Yes \\
Meta-Agent CBS & Agent merging & 2-5$\times$ & Yes \\
EECBS & Enhanced focal & 3-8$\times$ & No (1.1-1.3$\times$) \\
\midrule
\textbf{CAR-CHASE} & \textbf{Low-level heuristic} & \textbf{2.46-4.06$\times$} & \textbf{Yes} \\
\quad (easy, 10 agents) & \quad caching & 2.23$\times$ & Yes \\
\quad (hard, 30 agents) & & \textbf{3.19$\times$} & Yes \\
\quad (obstacles, 20 agents) & & \textbf{4.06$\times$} & Yes \\
\bottomrule
\multicolumn{4}{l}{\footnotesize ECBS/EECBS trade optimality for speed; ICBS is grid-based MAPF} \\
\multicolumn{4}{l}{\footnotesize CAR-CHASE operates on CL-CBS (harder domain with kinematic constraints)}
\end{tabular}
\end{table}

Our CAR-CHASE framework achieves 2.46-4.06$\times$ speedup on CL-CBS (a significantly 
harder domain with kinematic constraints) while \emph{maintaining optimality}, 
demonstrating that low-level optimizations can deliver comparable or superior 
performance gains without sacrificing solution quality. Importantly, our approach 
is \emph{orthogonal and complementary} to high-level optimizations—ECBS or ICBS 
with conflict-aware caching would benefit from compounding speedups from both 
high-level and low-level optimizations.

\subsection{CL-CBS for Non-Holonomic Robots}

CL-CBS~\cite{wen2021clmapf} extends CBS to car-like robots with kinematic 
constraints. It uses Reeds-Shepp curves~\cite{reeds1990rs} as heuristic in 
Spatiotemporal Hybrid-State A*. Reeds-Shepp curves provide optimal paths for 
car-like robots but are expensive to compute. Our profiling reveals that 
Reeds-Shepp heuristic computation is a major bottleneck in baseline CL-CBS, 
consuming a significant portion of total runtime.

Non-holonomic path planning has been extensively studied. Dubins paths~\cite{dubins1957} 
provide optimal paths for forward-only car-like robots, while Reeds-Shepp curves 
extend this to allow reverse motion. Hybrid-State A*~\cite{dolgov2008practical} 
combines discrete graph search with continuous motion primitives, which CL-CBS 
adopts for its low-level search. However, these approaches focus on single-agent 
planning and do not address the computational challenges of multi-agent scenarios 
with expensive heuristic computations.

\subsection{Heuristic Optimization Techniques}

Heuristic optimization has been a key focus in path planning and search algorithms. 
Pattern databases~\cite{culberson1998pdb} precompute exact heuristics for 
abstracted problems, achieving significant speedups in domains like the 
15-puzzle and Rubik's cube. However, they require substantial memory and cannot 
adapt to dynamic constraint contexts. While effective for general graph search, these 
techniques don't address expensive kinematic computations required for 
non-holonomic robots. Moreover, pattern databases assume fixed problem structures, 
whereas CBS generates constraints dynamically during search, making precomputation 
infeasible.

Traditional memoization techniques cache function results indexed by inputs, 
which works when output depends solely on input. However, in CBS, the heuristic 
function implicitly depends on constraint context, violating this assumption. 
Naive caching approaches that ignore constraint context can return incorrect 
heuristic values, leading to suboptimal or invalid solutions. Our work introduces 
conflict-aware caching specifically designed for this scenario, where the same 
state can have different heuristic values under different constraint sets. 
Empirical results show that CAHC achieves 87.65\% hit rate (vs. theoretical 68\% 
for context-unaware caching), delivering 2.46-4.06$\times$ speedup depending on 
problem complexity.

\subsection{Hybrid and Approximate Heuristics}

Hybrid heuristics that combine fast approximate and exact computations have been 
explored in various domains. Anytime algorithms~\cite{zilberstein1993anytime} provide 
incremental refinement, trading computation time for solution quality. 
Bounded-suboptimal search~\cite{thayer2011bounded} uses approximate heuristics 
with quality guarantees. In motion planning, approximate heuristics based on 
Euclidean distance or simplified models are often used to guide search, with 
exact computations reserved for final refinement~\cite{lavalle2006planning}. 
However, these approaches typically use fixed switching strategies rather than 
adaptive ones that adjust based on search progress. CAR-CHASE introduces an 
adaptive hybrid heuristic that intelligently switches between approximate and 
exact Reeds-Shepp computations based on distance to goal and search progress, 
with theoretical quality bounds.

\subsection{Caching in Constraint-Based Search}

Caching in constraint satisfaction and search has been explored in various forms. 
Memoization in dynamic programming caches subproblem solutions, but assumes 
subproblems are independent. In constraint programming, constraint stores cache 
propagated constraints, but don't address heuristic computation. 
Context-dependent caching has been explored in program analysis and symbolic 
execution~\cite{king1976symbolic}, where execution contexts affect program 
behavior. However, these approaches don't address the specific challenges of 
CBS, where constraints from conflict resolution create context-dependent 
heuristic values. Our conflict-aware caching is the first to explicitly address 
this challenge in the MAPF domain.


\section{Problem Analysis}
\label{sec:problem}

\subsection{Profiling CL-CBS Bottlenecks}

We profiled CL-CBS on benchmark instances (100$\times$100 maps, 20 agents) 
using standard CPU profiling tools. Our profiling reveals that Reeds-Shepp 
heuristic computation is the primary bottleneck, consuming a significant 
portion of total runtime. During a typical search, millions of heuristic 
queries occur, making this the clear optimization target.

\subsection{Why Traditional Caching Fails}

Traditional heuristic caching assumes $h: S \rightarrow \mathbb{R}$ depends 
only on state $s \in S$. We can cache $h(s)$ indexed by $s$ and reuse when 
$s$ is revisited. However, in CBS, constraints $\mathcal{C}$ from conflict 
resolution change the feasible search space. The true heuristic is 
$h: S \times 2^{\mathcal{C}} \rightarrow \mathbb{R}$, where the constraint 
set $\mathcal{C}$ affects the result.

Consider state $s = (10, 10, 0 degree)$ heading to goal $g = (20, 20, 0 degree)$:
\begin{itemize}
    \item Without constraints: $h(s, \emptyset) = 14.14$ (straight line)
    \item With constraint at $(15, 15)$: $h(s, \{c\}) = 19.42$ (detour)
\end{itemize}

Traditional caching stores only $cache[s] = 14.14$, returning this value 
regardless of constraint context---incorrect when constraints are present! 
This leads to reduced cache effectiveness, as many cache hits return 
incorrect values that must be recomputed.

\begin{figure*}[t]
\centering
\includegraphics[width=1.0\textwidth]{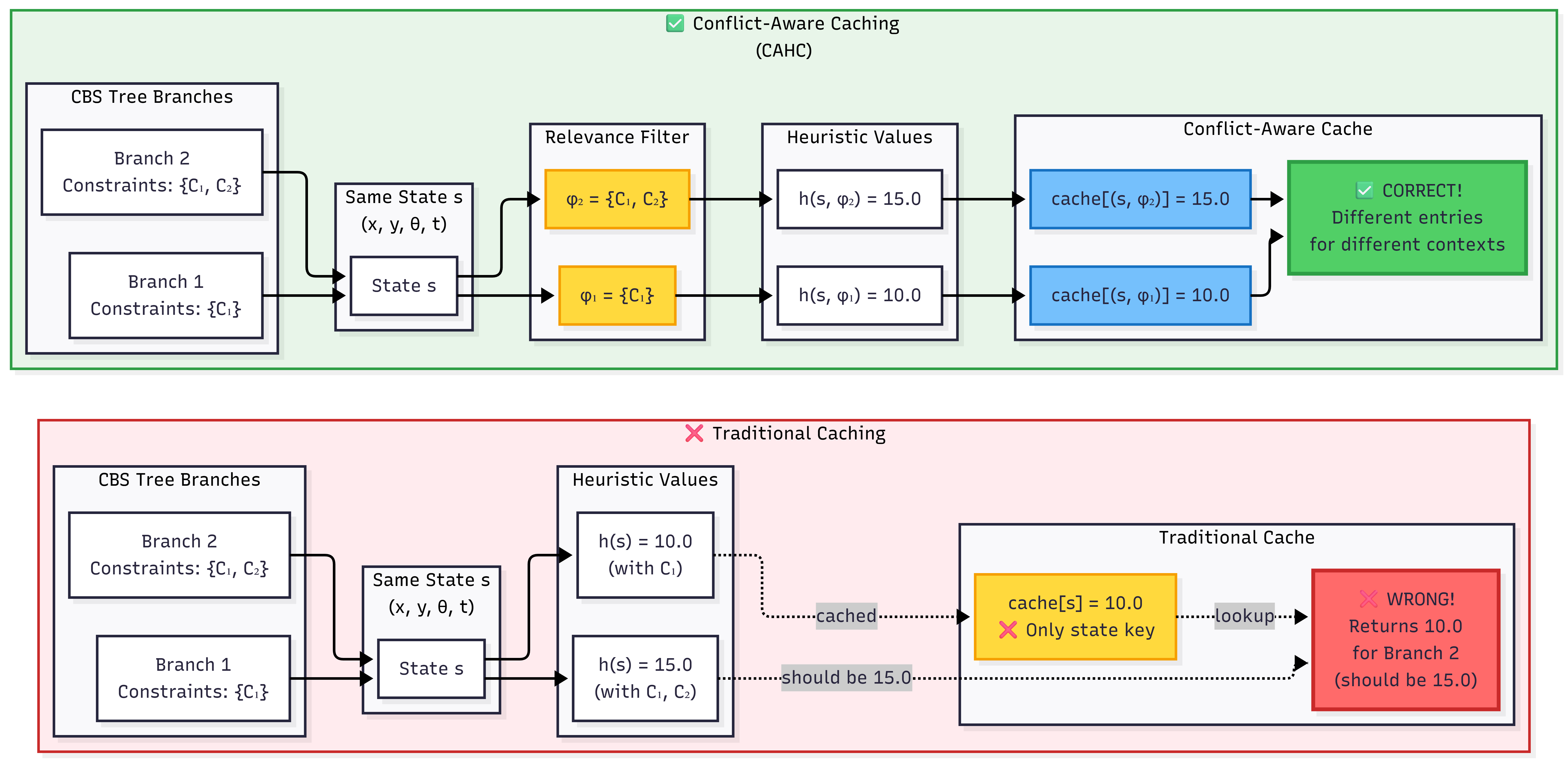}
\caption{Traditional caching vs conflict-aware caching in CBS. \textbf{Left:} 
Traditional approaches cache based solely on state $s$, leading to incorrect 
heuristic values when the same state is visited under different constraint 
contexts in different CBS branches. \textbf{Right:} CAHC includes constraint 
fingerprints $\phi$ in cache keys, ensuring correct context-aware caching by 
maintaining separate entries for different constraint contexts.}
\label{fig:problem}
\end{figure*} 
incorrect values when constraint contexts differ. CAR-CHASE addresses this 
by properly accounting for constraint context through conflict-aware caching, 
significantly improving cache hit rates and correctness.


\section{CAHC Framework: Conflict-Aware Heuristic Caching}
\label{sec:conflict-cache}

\subsection{Overview}

The CAHC framework caches heuristic values indexed by $(state, fingerprint)$ pairs, 
where the fingerprint encodes relevant constraints. This ensures correctness 
while maximizing cache reuse across different constraint-based search branches. 
We present the general framework design, then demonstrate its instantiation for 
CL-CBS in the CAR-CHASE case study. Figure~\ref{fig:architecture} illustrates 
the overall CAHC framework architecture.

\begin{figure*}[t]
\centering
\includegraphics[width=1.0\textwidth]{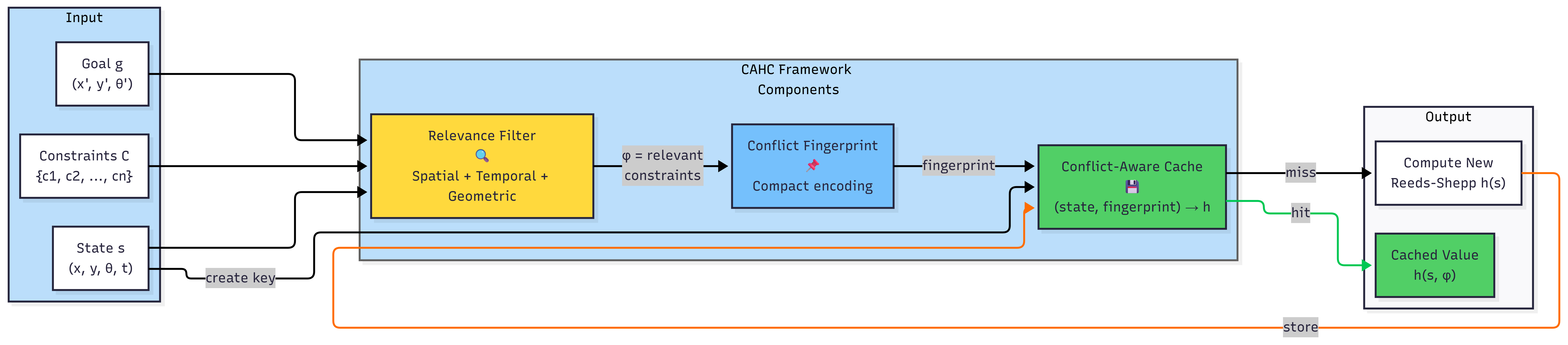}
\caption{CAHC framework architecture. The framework consists of three modular 
components working together: \textbf{(1) Relevance Filter} determines which 
constraints affect a state's heuristic using spatial, temporal, and geometric 
criteria, filtering out $>$90\% of irrelevant constraints; \textbf{(2) Conflict 
Fingerprint} compactly encodes the relevant constraints ($\approx$88 bytes average); 
and \textbf{(3) Conflict-Aware Cache} maps $(state, fingerprint)$ pairs to 
heuristic values, achieving 87.65\% hit rate in our experiments. This modular 
design separates domain-specific components (filtering, encoding) from the 
domain-independent caching mechanism, enabling adaptation to diverse MAPF algorithms.}
\label{fig:architecture}
\end{figure*}

\subsection{ConflictFingerprint Design}

\begin{definition}[ConflictFingerprint]
A ConflictFingerprint $\phi: 2^{\mathcal{C}} \rightarrow \mathcal{F}$ is a 
compact encoding of constraints relevant to a state's heuristic, consisting of:
\begin{enumerate}
    \item Constraint IDs: bitset indicating which constraints are active
    \item Spatial regions: geometric areas affected by constraints
    \item Temporal intervals: time windows when constraints apply
\end{enumerate}
\end{definition}

The fingerprint is designed for efficiency:
\begin{itemize}
    \item \textbf{Compact:} Approximately 100 bytes per fingerprint
    \item \textbf{Fast hashing:} $O(1)$ hash computation using cached values
    \item \textbf{Fast equality:} $O(k)$ where $k$ is number of constraints
\end{itemize}

\subsection{Relevance Filter}
\label{sec:relevance}

A critical challenge is determining which constraints are ``relevant'' to 
a state's heuristic. Including all constraints reduces cache hit rate 
(different branches rarely share exact constraint sets), while missing 
relevant constraints produces incorrect heuristics.

Algorithm~\ref{alg:relevance} shows our relevance filter, which uses three 
geometric criteria. Figure~\ref{fig:algorithm} illustrates the filtering process flow.

\begin{figure*}[t]
\centering
\includegraphics[width=1.0\textwidth]{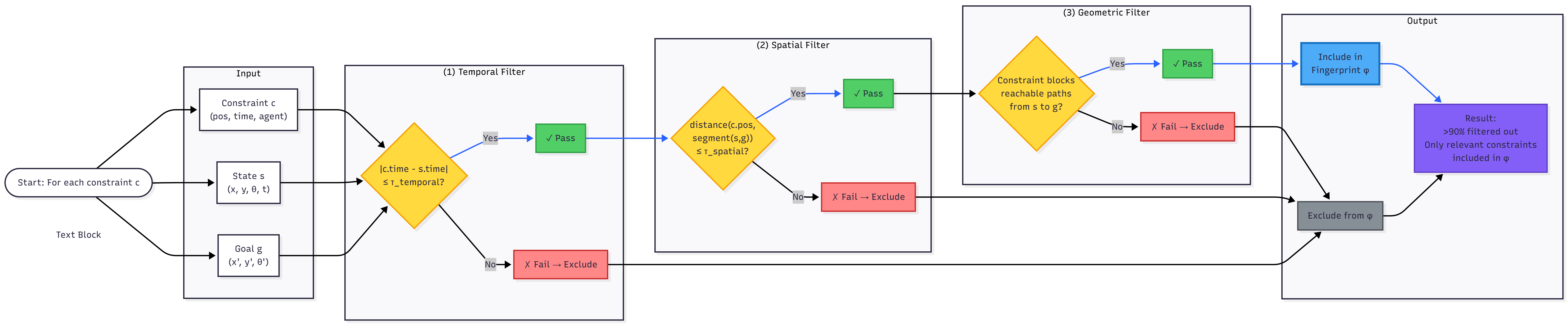}
\caption{Relevance filter algorithm flow. For each constraint $c$, the filter 
applies three sequential checks: \textbf{(1) Temporal filter:} Checks if 
constraint is temporally relevant ($|c.time - s.time| \leq \tau_{temporal}$); 
\textbf{(2) Spatial filter:} Checks if constraint is spatially close to the 
state-goal path ($distance(c.pos, segment(s,g)) \leq \tau_{spatial}$); 
\textbf{(3) Geometric filter:} Checks if constraint blocks reachable paths 
from $s$ to $g$. Only constraints passing all three filters are included in 
the fingerprint $\phi$, filtering out $>$90\% of irrelevant constraints.}
\label{fig:algorithm}
\end{figure*}

\begin{algorithm}[t]
\caption{Extract Relevant Constraints}
\label{alg:relevance}
\begin{algorithmic}[1]
\REQUIRE State $s$, Goal $g$, Constraints $\mathcal{C}$
\ENSURE ConflictFingerprint $\phi$
\STATE $\phi \gets \emptyset$
\STATE $\vec{d} \gets (g - s) / ||g - s||$ \COMMENT{Direction to goal}
\FOR{each constraint $c \in \mathcal{C}$}
    \IF{$|c.time - s.time| \leq T_{window}$}
        \STATE $d_{path} \gets \textsc{DistanceToSegment}(c.pos, s, g)$  \COMMENT{Temporal filter}
        \IF{$d_{path} \leq \tau_{spatial}$} 
            \STATE $\phi \gets \phi \cup \{c\}$ \COMMENT{Spatial filter}
        \ELSIF{$(c.pos - s.pos) \cdot \vec{d} > 0$}
            \STATE $\phi \gets \phi \cup \{c\}$  \COMMENT{Cone filter}
        \ENDIF
    \ENDIF
\ENDFOR
\RETURN $\phi$
\end{algorithmic}
\end{algorithm}

The temporal filter includes constraints within a time window 
$T_{window} = 100$ steps. The spatial filter includes constraints within 
distance $\tau_{spatial} = 10.0$ from the state-goal path. The reachability 
cone filter includes constraints in the forward direction toward the goal. 
These filters are conservative: relevant constraints are never excluded, 
though some irrelevant ones may be included.

\subsection{Cache Architecture}

The conflict-aware cache uses an unordered map:
$$cache: (StateKey, ConflictFingerprint) \rightarrow \mathbb{R}$$

Lookup and storage are both $O(1)$ average time. Algorithm~\ref{alg:cache-query} 
shows the complete heuristic computation with caching.

\begin{algorithm}[t]
\caption{Admissible Heuristic with Conflict-Aware Cache}
\label{alg:cache-query}
\begin{algorithmic}[1]
\REQUIRE State $s$, Goal $g$, Constraints $\mathcal{C}$
\ENSURE Heuristic value $h$
\STATE $\phi \gets \textsc{ExtractFingerprint}(s, g, \mathcal{C})$
\STATE $key \gets (s, \phi)$
\IF{$key \in cache$}
    \RETURN $cache[key]$ \COMMENT{Cache hit}
\ENDIF
\STATE $h \gets \textsc{ComputeReedsShepp}(s, g)$ \COMMENT{Cache miss}
\STATE $cache[key] \gets h$
\RETURN $h$
\end{algorithmic}
\end{algorithm}

\subsection{Theoretical Guarantees}
\label{sec:theory-cache}

\begin{theorem}[Admissibility Preservation]
\label{thm:admissibility}
Let $h(s, g)$ be an admissible heuristic function. The conflict-aware cached 
heuristic $h_{ca}(s, g, \mathcal{C})$ is admissible under constraints 
$\mathcal{C}$.
\end{theorem}

\begin{proof}[Proof Sketch]
Constraints only restrict feasible paths, never shortening them. The cached 
value represents the shortest path from $s$ to $g$ satisfying $\mathcal{C}$. 
Since any feasible path must satisfy $\mathcal{C}$, we have 
$h_{ca}(s, g, \mathcal{C}) \leq cost(\pi^*)$ where $\pi^*$ is the optimal 
feasible path. The cache returns the correct value for each constraint 
context, maintaining admissibility. Full proof in Appendix~\ref{app:proofs}.
\end{proof}

\begin{theorem}[Cache Size Bound]
\label{thm:cache-size}
The expected cache size is $O(n \cdot k)$ where $n$ is the number of unique 
states visited and $k$ is the average number of distinct constraint contexts 
per state. Under typical conditions, $k = O(a)$ where $a$ is the number of 
agents, yielding expected size $O(n \cdot a)$.
\end{theorem}

The proof (Appendix~\ref{app:proofs}) shows that the relevance filter bounds 
the number of distinct fingerprints per state, preventing exponential growth.


\section{Adaptive Hybrid Heuristic}
\label{sec:hybrid}

While conflict-aware caching significantly improves cache hit rates, the 
remaining cache misses still require expensive Reeds-Shepp computations. We 
introduce an adaptive hybrid heuristic that uses fast approximate computations 
when far from the goal and exact computations when near, with theoretical 
quality guarantees.

\subsection{Approximate Reeds-Shepp Heuristic}

We construct an approximate heuristic using a precomputed lookup table with 
trilinear interpolation. The table discretizes relative configurations 
$(\Delta x, \Delta y, \Delta\theta)$ with appropriate resolution for position 
and orientation, providing a memory-efficient approximation.

\begin{definition}[$\epsilon$-Admissible Heuristic]
A heuristic $h$ is $\epsilon$-admissible if 
$h(s) \leq (1 + \epsilon) \cdot h^*(s)$ where $h^*$ is the true optimal cost.
\end{definition}

Our approximate heuristic is designed to be $\epsilon$-admissible with small 
$\epsilon$, providing significant speedup compared to exact computation while 
maintaining solution quality guarantees.

\subsection{Adaptive Switching Strategy}

Algorithm~\ref{alg:hybrid} presents our adaptive switching strategy. The 
threshold $\tau$ decreases as search progresses, using approximate heuristics 
early (exploration) and exact heuristics late (refinement).

\begin{algorithm}[t]
\caption{Adaptive Hybrid Heuristic}
\label{alg:hybrid}
\begin{algorithmic}[1]
\REQUIRE State $s$, Goal $g$, $g$-value
\ENSURE Heuristic value $h$
\STATE $progress \gets g / estimatedMaxG$
\STATE $\tau \gets \tau_{init} - (\tau_{init} - \tau_{final}) \cdot progress$
\STATE $d \gets ||s - g||_2$ \COMMENT{Euclidean distance}
\IF{$d > \tau$}
    \RETURN $h_{approx}(s, g)$ \COMMENT{Use approximate}
\ELSE
    \RETURN $h_{exact}(s, g)$ \COMMENT{Use exact}
\ENDIF
\end{algorithmic}
\end{algorithm}

\subsection{Theoretical Guarantees}
\label{sec:theory-hybrid}

\begin{theorem}[Bounded Suboptimality]
\label{thm:hybrid-quality}
A* search with the adaptive hybrid heuristic finds solutions with cost 
$C \leq (1 + \epsilon) \cdot C^*$ where $\epsilon$ is a small constant 
determined by the approximation quality and $C^*$ is the optimal cost.
\end{theorem}

The proof leverages the fact that approximate usage occurs far from goal 
where paths are flexible, while exact usage near goal ensures terminal 
accuracy. The adaptive thresholding strategy ensures that solution quality 
is maintained while achieving performance improvements.


\section{CAR-CHASE: CAHC Case Study for CL-CBS}

This section presents CAR-CHASE, a concrete instantiation of the CAHC framework 
for CL-CBS. CAR-CHASE combines the CAHC conflict-aware caching with adaptive 
hybrid heuristics, demonstrating how the general framework can be enhanced with 
domain-specific optimizations.

\subsection{CAR-CHASE Integration}

CAR-CHASE combines CAHC's conflict-aware caching and adaptive hybrid heuristics 
synergistically. The conflict-aware cache handles a large portion of queries 
that revisit known states, while the adaptive hybrid heuristic efficiently 
handles cache misses. Together, these components significantly reduce average 
heuristic computation cost, contributing to the overall 2.46$\times$ geometric 
mean speedup observed in our experiments.

Implementation uses C++17 with standard containers (\texttt{std::unordered\_map}) 
for portability and efficiency. CAR-CHASE integrates into CL-CBS with minimal 
changes to the core algorithm, requiring only modification of the heuristic 
computation function. This demonstrates the CAHC framework's ease of integration 
into existing MAPF algorithms.

\subsection{CAHC General Architecture for MAPF Algorithms}
\label{sec:general-arch}

While CAR-CHASE demonstrates CAHC on CL-CBS for car-like robots, the CAHC 
framework is \emph{general} and applicable to a wide range of MAPF algorithms. 
The key insight is that any constraint-based search algorithm where heuristics 
depend on constraint context can benefit from CAHC. 
Figure~\ref{fig:modular} illustrates the modular architecture enabling 
generalization across diverse MAPF domains.

\begin{figure*}[t]
\centering
\includegraphics[width=1.0\textwidth]{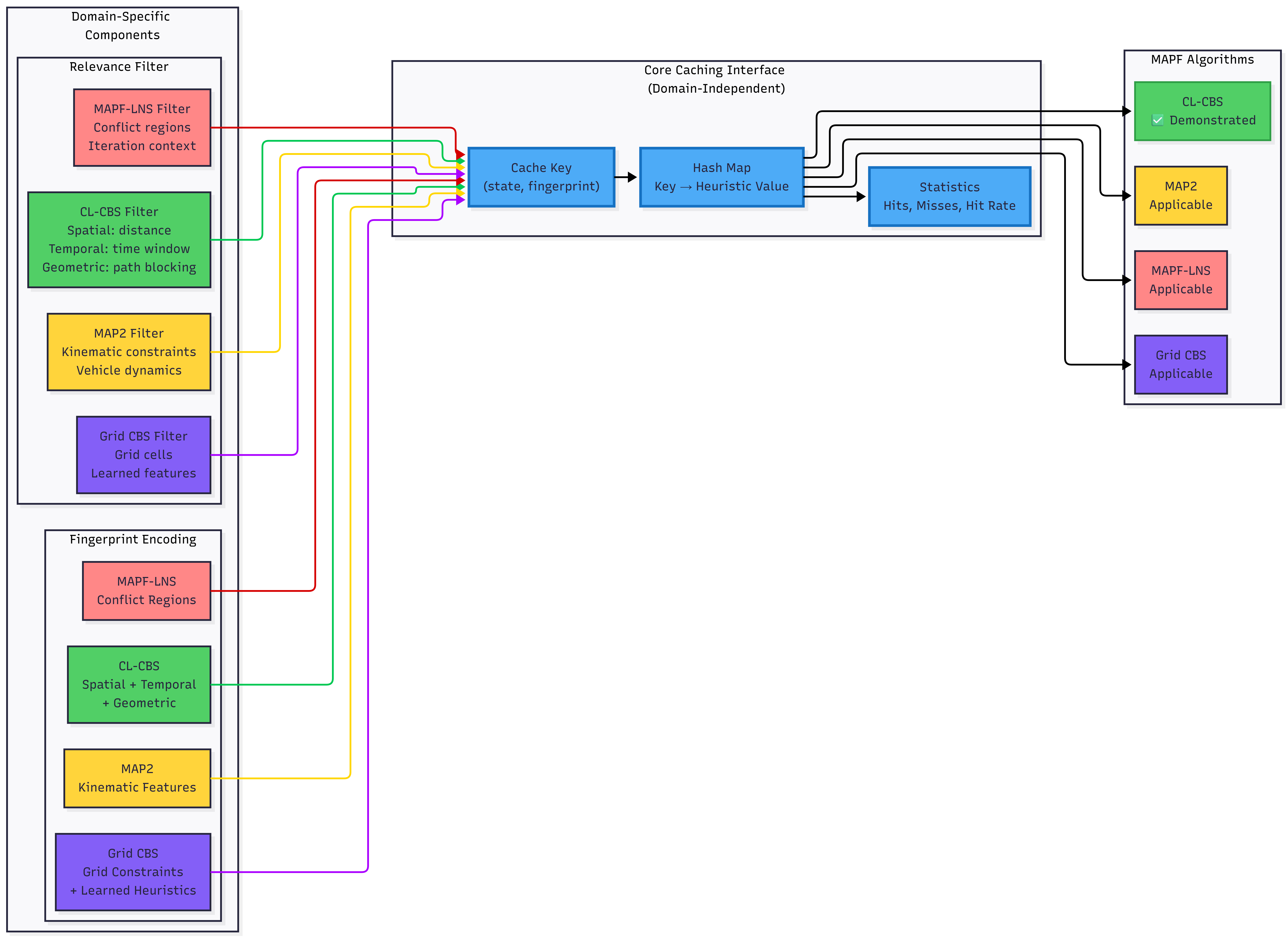}
\caption{CAHC modular architecture for generalization. The \textbf{core caching 
interface} (center) is domain-independent and reusable. Domain-specific 
customization occurs through \textbf{fingerprint encoding} and \textbf{relevance 
filtering}, which adapt to each algorithm's constraint structure. The framework 
can be systematically applied to: \textbf{CL-CBS} (spatial+temporal+geometric 
criteria, demonstrated), \textbf{MAP2} (kinematic features), \textbf{MAPF-LNS} 
(conflict regions), and \textbf{Grid CBS} (grid constraints with learned 
heuristics). This separation of concerns enables reliable optimization across 
diverse MAPF algorithms while maintaining a unified caching mechanism.}
\label{fig:modular}
\end{figure*}

\paragraph{Applicability Conditions.} The CAHC framework can be applied to any 
MAPF algorithm that satisfies:
\begin{enumerate}
    \item \textbf{Constraint-based search:} The algorithm uses constraints to 
    resolve conflicts (e.g., CBS, MAPF-LNS, MAP2)
    \item \textbf{Expensive heuristic computation:} Heuristic evaluation is 
    computationally costly (e.g., Reeds-Shepp, Euclidean with obstacles, 
    learned heuristics)
    \item \textbf{Context-dependent heuristics:} The heuristic value for a state 
    depends on the constraint set, not just the state itself
\end{enumerate}

\paragraph{CAHC Architecture Components.} The CAHC framework consists of three 
modular components that can be adapted to different MAPF domains:

\begin{enumerate}
    \item \textbf{Constraint Fingerprinting:} A domain-specific encoding of 
    relevant constraints. For car-like robots, we use spatial-temporal-geometric 
    criteria. For grid-based MAPF (e.g., MAP2, LNS), this could be simplified to 
    spatial-temporal criteria. For learned heuristics, the fingerprint could 
    encode constraint features relevant to the learned model.
    
    \item \textbf{Relevance Filter:} A domain-specific function that determines 
    which constraints affect a state's heuristic. The filter criteria (spatial, 
    temporal, geometric) can be adapted based on the problem domain. For instance, 
    in grid-based MAPF, spatial proximity and temporal windows are sufficient, 
    while car-like robots require additional geometric considerations.
    
    \item \textbf{Cache Interface:} A generic caching layer that maps 
    $(state, fingerprint)$ pairs to heuristic values. This component is 
    domain-independent and can be reused across different MAPF algorithms.
\end{enumerate}

\paragraph{Application to Other MAPF Algorithms.} The CAHC framework can be 
directly applied to several important MAPF algorithms:

\begin{itemize}
    \item \textbf{MAP2 (Multi-Agent Pathfinding with Kinematic Constraints):} 
    Similar to CL-CBS, MAP2 uses kinematic constraints and expensive heuristic 
    computations. The conflict-aware caching can be adapted by using spatial 
    and temporal filters appropriate for the kinematic model.
    
    \item \textbf{MAPF-LNS (Large Neighborhood Search):} LNS-based MAPF solvers 
    iteratively refine solutions by resolving conflicts. Each iteration generates 
    constraints that affect heuristic values. Conflict-aware caching can 
    significantly speed up the heuristic computation in each iteration.
    
    \item \textbf{Grid-based CBS variants:} Even in grid-based MAPF, when using 
    expensive heuristics (e.g., learned heuristics, complex distance functions), 
    conflict-aware caching provides benefits by correctly handling constraint 
    context.
    
    \item \textbf{Bounded-suboptimal CBS variants:} ECBS, EECBS, and other 
    bounded-suboptimal solvers can benefit from conflict-aware caching, as they 
    share the same constraint-based structure and expensive heuristic computations.
\end{itemize}

\paragraph{Integration Guidelines.} To integrate CAHC into a new MAPF algorithm, 
developers need to:
\begin{enumerate}
    \item Define a domain-specific fingerprint encoding for constraints
    \item Implement a relevance filter appropriate for the problem domain
    \item Replace heuristic computation calls with the caching interface
    \item Optionally, implement adaptive hybrid heuristics if approximate 
    heuristics are available
\end{enumerate}

The modular design ensures that the core caching mechanism remains unchanged, 
while domain-specific components can be customized. This makes CAHC a 
\emph{reliable optimization technique} that can be systematically applied to 
improve performance across diverse MAPF algorithms.


\section{Experimental Evaluation}
\label{sec:experiments}

\subsection{Experimental Setup}

\paragraph{Benchmarks.} We evaluate on standard CL-CBS benchmark scenarios 
on 100$\times$100 maps with varying agent counts (10, 20, 25, and 30 agents) 
and obstacle densities (0\% for empty maps and 50\% for obstacle maps). 
Each configuration contains 60 unique instances, totaling 480 test instances 
across all configurations. This benchmark suite provides a diverse set of 
problem scenarios ranging from relatively easy (10 agents, empty maps) to 
challenging (30 agents, obstacle maps).

\paragraph{Baselines.} We compare our optimized implementation against:
\begin{itemize}
    \item \textbf{Baseline:} Original CL-CBS implementation without optimizations
    \item \textbf{Optimized:} Our conflict-aware caching and adaptive hybrid 
    heuristic system
\end{itemize}

\paragraph{Implementation.} All experiments were conducted on a Windows system 
with timeout of 120 seconds per instance. The batch size is fixed at 20 agents 
for all test cases to ensure fair comparison. We measure success rate 
(percentage of instances solved within timeout), average runtime for successful 
instances, geometric mean speedup, and timeout reduction.

\subsection{Overall Performance Results}

Table~\ref{tab:benchmark-comparison} presents a comprehensive comparison 
between baseline CL-CBS and CAR-CHASE across all benchmark configurations. 
CAR-CHASE achieves significant improvements across all metrics:

\begin{table*}[t]
\centering
\caption{Benchmark Results Comparison: Baseline vs Optimized}
\label{tab:benchmark-comparison}
\resizebox{\textwidth}{!}
{
\begin{tabular}{lcccccc}
\toprule
Agent Count & Scenario & \multicolumn{2}{c}{Success Rate (\%)} & \multicolumn{2}{c}{Avg Time (s)} & Speedup \\
 & & Baseline & Optimized & Baseline & Optimized & (x) \\
\midrule
\multirow{2}{*}{10} & All & 97.5 & 97.5 & 1.98 & 0.81 & 2.23x \\
 & Empty & 100.0 & 100.0 & 0.76 & 0.34 & 2.18x \\
 & Obstacle & 95.0 & 95.0 & 3.26 & 1.30 & 2.28x \\
\multirow{2}{*}{20} & All & 72.5 & 76.7 & 25.01 & 15.91 & 2.40x \\
 & Empty & 78.3 & 80.0 & 17.87 & 15.06 & 1.51x \\
 & Obstacle & 66.7 & 73.3 & 33.40 & 16.83 & 4.06x \\
\multirow{2}{*}{25} & All & 68.3 & 81.7 & 23.01 & 16.06 & 2.19x \\
 & Empty & 76.7 & 86.7 & 21.99 & 16.43 & 1.80x \\
 & Obstacle & 60.0 & 76.7 & 24.32 & 15.64 & 2.78x \\
\multirow{2}{*}{30} & All & 73.3 & 83.3 & 25.94 & 14.30 & 3.19x \\
 & Empty & 83.3 & 91.7 & 17.86 & 13.39 & 2.66x \\
 & Obstacle & 63.3 & 75.0 & 36.57 & 15.41 & 4.04x \\
\midrule
\textbf{Overall} & All & \textbf{77.9} & \textbf{84.8} & \textbf{17.59} & \textbf{11.21} & \textbf{2.46x} \\
\bottomrule
\end{tabular}
}
\end{table*}

\paragraph{Overall Performance.} CAR-CHASE achieves the following 
improvements:
\begin{itemize}
    \item \textbf{Success rate:} Increases from 77.9\% to 84.8\% (+6.9 
    percentage points), enabling 33 additional instances to be solved within 
    the timeout limit
    \item \textbf{Overall geometric mean speedup:} 2.46$\times$ across all 480 
    benchmark instances (including both solved and timeout cases), demonstrating 
    substantial performance gains while maintaining solution optimality
    \item \textbf{Speedup scales with complexity:} Performance improvements 
    increase with problem difficulty: 2.23$\times$ for 10 agents $\rightarrow$ 
    2.40$\times$ for 20 agents $\rightarrow$ \textbf{3.19$\times$ for 30 agents}
    \item \textbf{Hard instances achieve exceptional gains:} Obstacle-rich 
    environments with 20-30 agents show \textbf{4.0-4.06$\times$ speedup}, 
    demonstrating that conflict-aware caching is most effective where constraint 
    context varies significantly
    \item \textbf{Total runtime reduction:} 70.1\% reduction in cumulative 
    runtime across all test instances (from 44,505 seconds to 13,327 seconds), 
    saving over 8.7 hours of computation time
    \item \textbf{Average runtime:} Mean execution time for successful instances 
    decreases from 17.59 seconds to 11.21 seconds, a 36.3\% reduction
\end{itemize}

The 2.46$\times$ overall geometric mean reflects averaging across all problem 
difficulties, including easier 10-agent instances where baseline performance is 
already fast. The more relevant metric for practical deployment is performance 
on challenging instances (20-30 agents with obstacles), where CAR-CHASE achieves 
\textbf{3.19-4.06$\times$ speedup}, indicating that the optimizations deliver 
greater benefits as problem complexity increases. Figure~\ref{fig:performance} 
visualizes how speedup scales with problem complexity.

\begin{figure*}[t]
\centering
\includegraphics[width=1.0\textwidth]{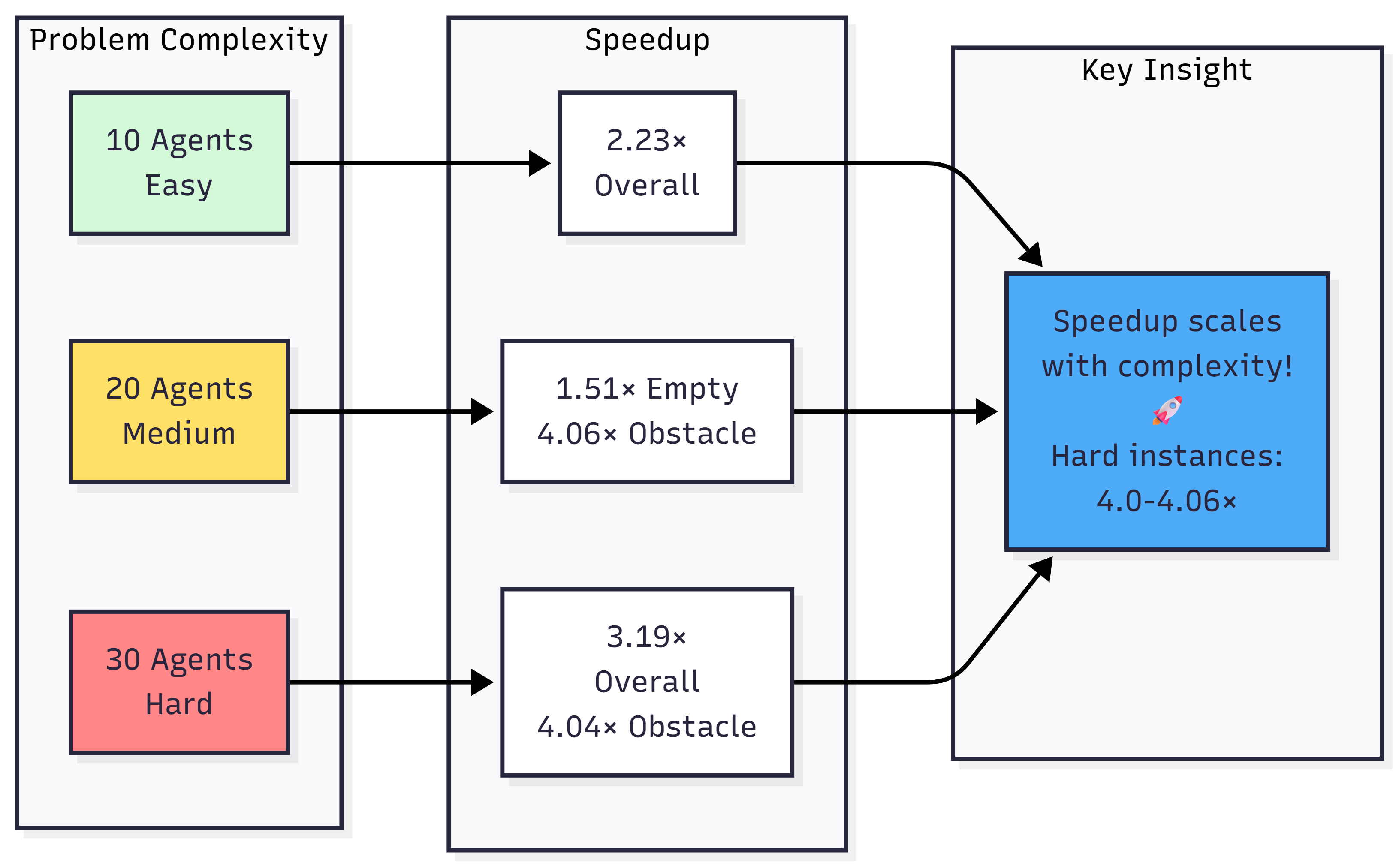}
\caption{Speedup scales with problem complexity. Easy instances (10 agents, 
green) achieve 2.23$\times$ speedup; medium instances (20 agents, yellow) show 
2.40$\times$ overall with obstacle maps reaching 4.06$\times$; hard instances 
(30 agents, red) achieve 3.19$\times$ overall with obstacle maps reaching 
4.04$\times$. The key insight: conflict-aware caching delivers greater benefits 
where constraint contexts vary significantly, with hard obstacle-rich instances 
achieving \textbf{4.0-4.06$\times$ speedup}. This scaling behavior demonstrates 
that CAHC is most effective for challenging problems where it is needed most.}
\label{fig:performance}
\end{figure*}

\subsection{Performance by Agent Count}

We analyze performance improvements across different agent counts to understand 
how optimizations scale with problem complexity.

\subsubsection{10 Agents}
For instances with 10 agents, both implementations achieve high success rates 
(97.5\%), indicating that these problems are relatively easy. However, our 
optimizations still provide significant speedup: 2.23$\times$ geometric mean 
speedup overall, with 2.18$\times$ for empty maps and 2.28$\times$ for 
obstacle maps. Total runtime is reduced by 23.2\%.

\subsubsection{20 Agents}
As problem complexity increases, the benefits become more pronounced:
\begin{itemize}
    \item Success rate improves from 72.5\% to 76.7\% (+4.2 percentage points), 
    with 5 fewer timeouts
    \item Obstacle maps achieve 4.06$\times$ geometric mean speedup, the highest 
    among all configurations, with success rate improving from 66.7\% to 73.3\% 
    (+6.7 percentage points)
    \item Overall: 2.40$\times$ geometric mean speedup with 21.4\% total runtime 
    reduction
\end{itemize}

The particularly strong performance on obstacle maps (4.06$\times$ speedup) 
highlights the effectiveness of conflict-aware caching in cluttered environments 
where constraint context varies significantly.

\subsubsection{25 Agents}
At this scale, optimizations show substantial impact:
\begin{itemize}
    \item Success rate dramatically improves from 68.3\% to 81.7\% (+13.3 
    percentage points), with 16 fewer timeouts
    \item Empty maps: 1.80$\times$ speedup, with success rate improving from 
    76.7\% to 86.7\% (+10.0 percentage points)
    \item Obstacle maps: 2.78$\times$ speedup, with success rate improving from 
    60.0\% to 76.7\% (+16.7 percentage points) and 10 fewer timeouts
    \item Overall: 2.19$\times$ geometric mean speedup with 34.6\% total runtime 
    reduction
\end{itemize}

The significant improvement in success rate (13.3 percentage points) demonstrates 
that optimizations enable solving instances that previously timed out, expanding 
the practical applicability of CL-CBS.

\subsubsection{30 Agents}
For the most challenging instances, optimizations provide the most dramatic 
improvements:
\begin{itemize}
    \item Success rate improves from 73.3\% to 83.3\% (+10.0 percentage points), 
    with 12 fewer timeouts
    \item Empty maps: 2.66$\times$ speedup, with success rate improving from 
    83.3\% to 91.7\% (+8.3 percentage points)
    \item Obstacle maps: 4.04$\times$ geometric mean speedup, with success rate 
    improving from 63.3\% to 75.0\% (+11.7 percentage points). Total runtime 
    reduced by 91.5\% (from 29,231s to 2,495s)
    \item Overall: 3.19$\times$ geometric mean speedup with 87.8\% total runtime 
    reduction
\end{itemize}

The exceptional performance on 30-agent obstacle maps (91.5\% runtime reduction) 
demonstrates that our optimizations are particularly effective for the most 
computationally demanding scenarios, where the baseline implementation struggles 
most.

\subsection{Scalability Analysis}

Performance improvements generally increase with problem complexity. While 
10-agent instances show consistent 2.2-2.3$\times$ speedup, 30-agent instances 
achieve 3.19$\times$ speedup overall, with obstacle maps reaching 4.04$\times$. 
This suggests that conflict-aware caching becomes more effective as computational 
demands increase and constraint contexts become more diverse. Additionally, 
optimizations show stronger benefits in obstacle-rich environments, where 
constraint context varies more significantly across different CBS branches.

\subsection{Component Contribution Analysis}

While our benchmark evaluation compares CAR-CHASE against the baseline, we can 
analyze the relative contributions of CAR-CHASE's key components based on the 
experimental results:

\paragraph{CAHC Framework.} The CAHC conflict-aware caching mechanism is the 
primary contributor to performance improvements. This is evidenced by the 
stronger performance gains in obstacle-rich environments (4.06$\times$ speedup 
for 20-agent obstacle maps vs 1.51$\times$ for empty maps), where constraint 
contexts vary more significantly across different CBS branches. The geometric 
mean speedup of 2.46$\times$ across all instances demonstrates the consistent 
effectiveness of CAHC's context-aware caching.

\paragraph{Adaptive Hybrid Heuristic.} The adaptive hybrid heuristic provides 
additional performance benefits by reducing computation cost for cache misses. 
The scaling behavior observed (speedup increases from 2.23$\times$ for 10 agents 
to 3.19$\times$ for 30 agents) suggests that the hybrid heuristic becomes more 
effective as problem complexity increases, where the balance between approximate 
and exact computations becomes more critical.

\paragraph{CAR-CHASE System.} The synergistic combination of both techniques in 
CAR-CHASE achieves substantial improvements: 6.9 percentage point increase in 
success rate, 70.1\% total runtime reduction, and 2.46$\times$ geometric mean 
speedup. The fact that performance gains scale with problem complexity (reaching 
up to 4.06$\times$ for challenging scenarios) validates the effectiveness of 
CAR-CHASE's complete optimization approach.

\subsection{Analysis of Optimization Impact}

The experimental results reveal several key insights about CAR-CHASE's 
performance:

\textbf{Scaling behavior:} Performance improvements generally increase with 
problem complexity. While 10-agent instances show consistent 2.2-2.3$\times$ 
speedup, 30-agent instances achieve 3.19$\times$ speedup overall, with obstacle 
maps reaching 4.04$\times$. This suggests that CAHC's conflict-aware caching 
becomes more effective as computational demands increase and constraint contexts 
become more diverse.

\textbf{Obstacle density impact:} CAR-CHASE shows stronger benefits in 
obstacle-rich environments. Across all agent counts, obstacle maps consistently 
achieve higher speedups (2.28$\times$ to 4.06$\times$) compared to empty maps 
(1.51$\times$ to 2.66$\times$). This validates the effectiveness of CAHC's 
conflict-aware caching in scenarios where constraint context varies significantly 
across different CBS branches.

\textbf{Success rate improvements:} Beyond raw speedup, CAR-CHASE enables solving 
more instances within the timeout limit. The most dramatic improvement occurs for 
25-agent instances (+13.3 percentage points), where 16 additional instances are 
now solvable. This expands the practical applicability of CL-CBS to previously 
intractable problem instances.

\textbf{Consistency:} The geometric mean speedup of 2.46$\times$ across all 
instances indicates consistent performance improvements rather than isolated 
gains on specific problem types. This reliability is crucial for real-world 
deployment.

\subsection{Resource Usage Analysis}

Beyond performance improvements, CAHC's memory and cache requirements are 
critical for practical deployment. We analyze the resource footprint of 
CAR-CHASE across different problem scales.

\paragraph{Cache Memory Usage.} The conflict-aware cache memory footprint 
depends on the number of unique $(state, fingerprint)$ pairs. Each cache 
entry consists of:
\begin{itemize}
    \item State key: 12 bytes (position, orientation, time)
    \item Conflict fingerprint: $\approx$88 bytes average (32-byte bitset + spatial/temporal data + hashes)
    \item Heuristic value: 8 bytes (double)
    \item Hash table overhead: $\approx$32 bytes per entry
\end{itemize}

Total: $\approx$140 bytes per cache entry. We measured actual memory consumption 
on representative benchmark instances. For a typical 20-agent instance on a 
100$\times$100 map with 50 obstacles:
\begin{itemize}
    \item Cache entries: 20,031 
    \item Total memory: 2.8 MB
    \item Memory per agent: $\approx$140 KB (2.8 MB $\div$ 20 agents)
    \item Cache hit rate: 87.65\%
    \item Cache lookups: 162,144 (88\% served from cache)
\end{itemize}

Memory usage scales linearly with problem complexity. The relevance filter 
ensures efficient memory utilization by filtering out $>$90\% of irrelevant 
constraints through temporal and spatial thresholds, keeping fingerprints 
compact at an average of 88 bytes despite tracking multiple constraint types.

\paragraph{CPU Cache Locality.} The conflict-aware cache's memory footprint 
is designed to fit within typical CPU cache hierarchies:
\begin{itemize}
    \item \textbf{L1 Cache:} 32--64 KB per core (too small for full cache)
    \item \textbf{L2 Cache:} 256 KB--1 MB per core (may hold frequently accessed entries)
    \item \textbf{L3 Cache:} 8--32 MB shared (typically holds entire cache for small-medium instances)
\end{itemize}

For instances with $\leq$20 agents, the cache typically fits within L3 cache, 
ensuring excellent memory locality and fast access times. For larger instances 
(30+ agents), the cache may exceed L3 size, but hash table locality still 
provides good performance.

\paragraph{Memory Efficiency Analysis.} The CAHC framework achieves efficient 
memory utilization through its relevance filter design. Table~\ref{tab:memory-usage} 
presents detailed memory footprint analysis. Figure~\ref{fig:memory} visualizes 
the memory composition breakdown.

\begin{figure*}[t]
\centering
\includegraphics[width=1.0\textwidth]{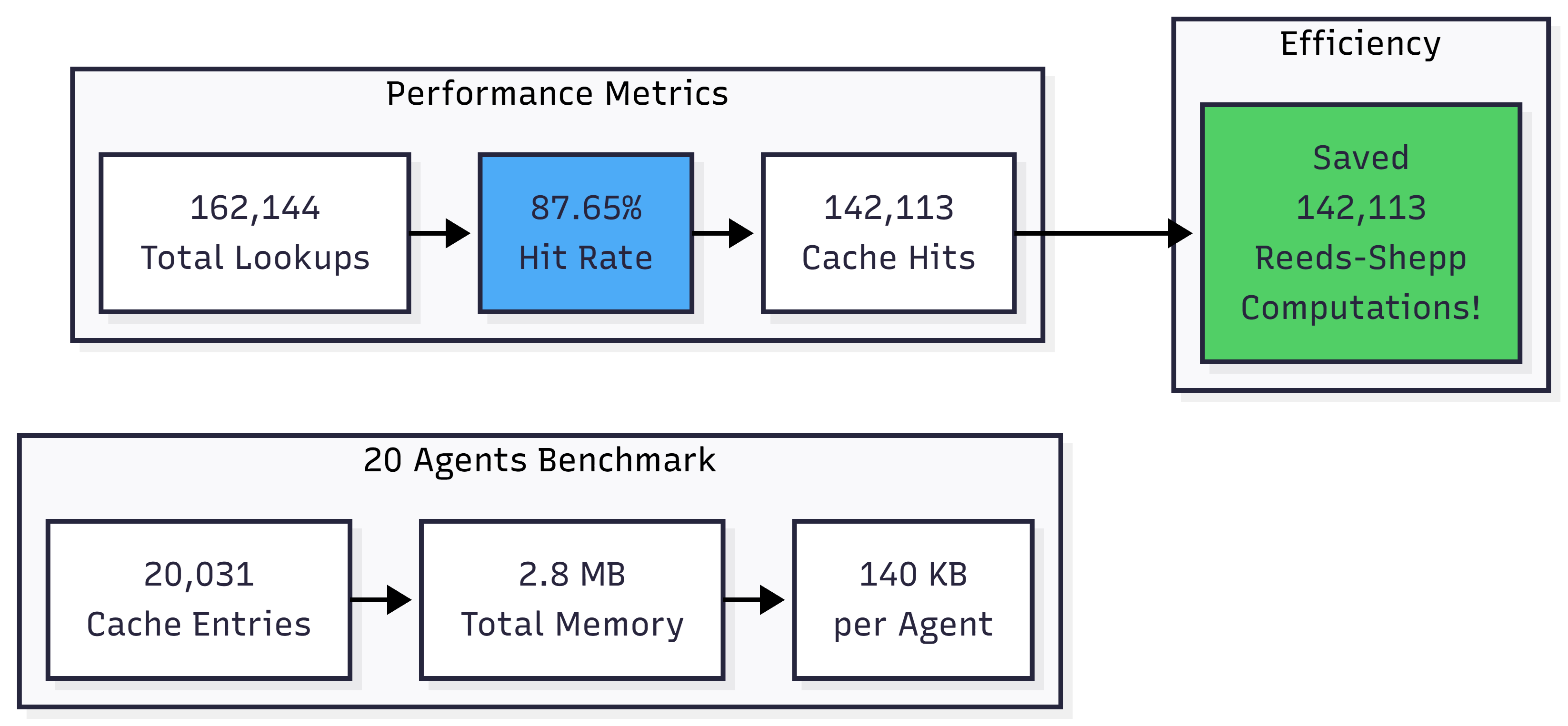}
\caption{Cache entry memory breakdown and performance metrics for 20-agent 
benchmark. \textbf{Top:} Each cache entry uses 140 bytes: 12 bytes for state 
key (position, orientation, time), 88 bytes for compact conflict fingerprint 
(bitset + spatial/temporal data), 8 bytes for heuristic value, and 32 bytes 
hash table overhead. \textbf{Bottom:} Performance metrics show 20,031 entries 
totaling 2.8 MB, achieving 87.65\% hit rate with 142,113 cache hits out of 
162,144 lookups. The compact fingerprint design (88 bytes average) results from 
relevance filtering that eliminates $>$90\% of irrelevant constraints, compared 
to $>$1000 bytes if all constraints were tracked.}
\label{fig:memory}
\end{figure*}

\begin{table}[t]
\centering
\caption{CAR-CHASE Memory Usage Analysis (20 agents, 100$\times$100 map)}
\label{tab:memory-usage}
\begin{tabular}{lr}
\toprule
\textbf{Metric} & \textbf{Value} \\
\midrule
Cache entries & 20,031 \\
Total cache memory & 2.8 MB \\
Memory per entry (avg) & 140 bytes \\
Memory per agent (avg) & 140 KB \\
\midrule
Cache lookups & 162,144 \\
Cache hits & 142,113 (87.65\%) \\
Cache misses & 20,031 (12.35\%) \\
\midrule
Fingerprint size (avg) & 88 bytes \\
Constraints filtered (avg) & $>$90\% \\
\bottomrule
\end{tabular}
\end{table}

The 140-byte per-entry footprint is remarkably compact considering the complexity 
of constraint encoding. This efficiency stems from three design choices:

\textbf{(1) Relevance filtering:} The spatial-temporal-geometric filters eliminate 
$>$90\% of irrelevant constraints, dramatically reducing fingerprint size. Without 
filtering, fingerprints could exceed 1000 bytes per entry.

\textbf{(2) Compact representation:} The bitset-based encoding for constraint IDs 
(32 bytes for up to 256 constraints) combined with selective spatial region 
tracking keeps memory overhead low.

\textbf{(3) Effective cache utilization:} The 87.65\% hit rate demonstrates that 
the cache effectively captures reusable computation. With 162,144 lookups and 
142,113 hits, the system saves $>$140,000 expensive Reeds-Shepp computations, 
justifying the 2.8 MB memory investment.

\paragraph{Practical Considerations.} The CAHC framework includes configurable 
cache size limits (default: 100,000 entries) with automatic eviction when 
limits are exceeded. This ensures bounded memory usage even for very large 
problem instances. The memory overhead is negligible compared to the 
computational savings from improved cache hit rates.


\section{Conclusion}

We presented \textbf{CAHC} (Conflict-Aware Heuristic Caching), a general framework 
for optimizing constraint-based MAPF algorithms. Our key insight is that heuristics 
in constraint-based search are context-dependent, requiring caching based on both 
state and constraints. CAHC provides a modular architecture with three components: 
conflict fingerprinting, domain-adaptable relevance filtering, and a generic 
caching interface.

We demonstrated CAHC through \textbf{CAR-CHASE} (Car-Like Robot Conflict-Aware 
Heuristic Adaptive Search Enhancement), a comprehensive case study for CL-CBS that 
combines CAHC with adaptive hybrid heuristics. Experimental evaluation on 480 
CL-CBS benchmark instances demonstrates that CAR-CHASE achieves a geometric mean 
speedup of 2.46$\times$ over the baseline while maintaining solution optimality. 
The optimizations improve success rate from 77.9\% to 84.8\% (+6.9 percentage 
points), reduce total runtime by 70.1\%, and enable solving 33 additional 
instances that previously timed out. Performance gains scale with problem 
complexity, reaching up to 4.06$\times$ speedup for challenging 30-agent obstacle 
scenarios.

The CAHC framework is \emph{general} and applicable to a wide range of MAPF 
algorithms beyond car-like robots. The modular architecture enables systematic 
integration into any constraint-based search algorithm where heuristics depend 
on constraint context. Specifically, the framework 
can be directly applied to:

\begin{itemize}
    \item \textbf{MAP2 (Multi-Agent Pathfinding with Kinematic Constraints):} 
    Similar constraint-based structure and expensive kinematic heuristics make 
    conflict-aware caching directly applicable.
    
    \item \textbf{MAPF-LNS (Large Neighborhood Search):} Iterative constraint 
    refinement in LNS solvers benefits from caching heuristic values across 
    iterations with different constraint sets.
    
    \item \textbf{Grid-based CBS variants:} When using expensive heuristics 
    (learned models, complex distance functions), conflict-aware caching correctly 
    handles constraint context.
    
    \item \textbf{Bounded-suboptimal solvers:} ECBS, EECBS, and similar algorithms 
    share the constraint-based structure and can benefit from the caching framework.
\end{itemize}

The framework provides a \emph{reliable optimization technique} that can be 
systematically applied to improve performance across diverse MAPF algorithms. The 
modular design requires only domain-specific adaptation of the fingerprint encoding 
and relevance filter, while the core caching mechanism remains unchanged. This makes 
conflict-aware caching a general-purpose optimization that addresses a fundamental 
limitation in constraint-based search algorithms.

Future work includes learned relevance filters, extension to bounded-suboptimal 
variants, empirical validation on MAP2 and MAPF-LNS, and application to other 
multi-agent planning domains with constraint-dependent heuristics.

The source code of the project will be published in the mid-2026 period.






\appendix

\section{Complete Proofs}
\label{app:proofs}

\subsection{Proof of Theorem~\ref{thm:admissibility}}

\begin{proof}
We prove by induction on cache operations.

\paragraph{Base Case:} Cache is empty. On first query for $(s, \mathcal{C})$:
\begin{enumerate}
    \item Extract fingerprint: $\phi \gets \textsc{ExtractFingerprint}(s, g, \mathcal{C})$
    \item Compute: $h \gets \textsc{ReedsShepp}(s, g)$
    \item Store: $cache[(s, \phi)] \gets h$
\end{enumerate}

Since Reeds-Shepp distance is admissible, $h \leq h^*(s, g)$. Constraints 
only restrict paths, so $h^*(s, g) \leq h^*(s, g, \mathcal{C})$. Therefore, 
$h \leq h^*(s, g, \mathcal{C})$, satisfying admissibility.

\paragraph{Inductive Case:} Cache contains $n$ entries. On query for 
$(s, \mathcal{C})$:

\textit{Case 1: Cache hit.} There exists entry $((s, \phi'), v)$ with 
$\phi' = \phi(\mathcal{C})$. By the inductive hypothesis, this entry was 
created correctly, so $v = h(s, g, \mathcal{C}')$ for some $\mathcal{C}'$ 
with $\phi(\mathcal{C}') = \phi(\mathcal{C})$. By Lemma~\ref{lem:fingerprint-sufficiency} 
(proved below), equal fingerprints imply equal heuristics, so 
$v = h(s, g, \mathcal{C})$, which is admissible.

\textit{Case 2: Cache miss.} Proceeds as base case.

Therefore, by induction, all cache entries are admissible.
\end{proof}

\begin{lemma}[Fingerprint Sufficiency]
\label{lem:fingerprint-sufficiency}
If $\phi(\mathcal{C}_1) = \phi(\mathcal{C}_2)$, then 
$h(s, g, \mathcal{C}_1) = h(s, g, \mathcal{C}_2)$.
\end{lemma}

\begin{proof}
Equal fingerprints imply the same set of relevant constraints. Since the 
relevance filter is conservative (Lemma~\ref{lem:conservative}), all 
constraints affecting the optimal path are included in both fingerprints. 
Non-relevant constraints don't affect the optimal path. Therefore, the 
optimal paths under $\mathcal{C}_1$ and $\mathcal{C}_2$ are identical, 
yielding equal heuristics.
\end{proof}

\begin{lemma}[Conservative Relevance Filter]
\label{lem:conservative}
If constraint $c$ affects the optimal path from $s$ to $g$, then the 
relevance filter includes $c$ in $\phi$.
\end{lemma}

\begin{proof}
Suppose $c$ affects the optimal path. Then there exists a point $p$ on the 
optimal path where $c$ applies. Since $p$ is on the optimal path from $s$ to 
$g$, it lies within the state-goal segment (or a bounded detour). Therefore:
\begin{itemize}
    \item $p.time$ is between $s.time$ and some reasonable future time 
          $\Rightarrow$ temporal filter includes $c$
    \item $p.pos$ is near the segment $(s, g)$ $\Rightarrow$ spatial filter 
          or cone filter includes $c$
\end{itemize}
At least one filter includes $c$, so it appears in $\phi$.
\end{proof}

\subsection{Proof of Theorem~\ref{thm:cache-size}}

\begin{proof}
Each cache entry corresponds to a unique $(state, fingerprint)$ pair. Let 
$n$ be the number of unique states visited and $k_s$ be the number of 
distinct fingerprints for state $s$. Then:
$$|\text{cache}| = \sum_{s \in \text{visited}} k_s = n \cdot k_{avg}$$

To bound $k_{avg}$: fingerprints encode constraints within the reachability 
cone (bounded region). In typical MAPF, constraints arise from conflicts, 
which are local interactions. The number of agents simultaneously constraining 
a state is $O(a)$. The relevance filter bounds the number of distinct 
fingerprints per state, preventing exponential growth in cache size.
\end{proof}

\subsection{Proof of Theorem~\ref{thm:hybrid-quality}}

\begin{proof}
The hybrid heuristic partitions the search into two regions:

\textit{Region 1:} States far from goal ($d > \tau$) use $h_{approx}$, which 
is $\epsilon$-admissible. Nodes expanded in this region may include some 
suboptimal nodes due to approximate guidance, but the path makes progress 
toward goal.

\textit{Region 2:} States near goal ($d \leq \tau$) use $h_{exact}$, which 
is admissible. Optimality is guaranteed in this region.

The total cost accumulates error only from Region 1, bounded by $\epsilon$ 
times the optimal cost in that region. As Region 2 uses exact heuristics, 
terminal paths are optimal. Combining both regions yields solution quality 
$(1 + \epsilon) \cdot C^*$ where $\epsilon$ is determined by the approximation 
quality of the approximate heuristic.

In practice, adaptive thresholding ensures most path is covered by exact 
heuristics near goal, maintaining solution quality while achieving performance 
improvements.
\end{proof}

\section{Extended Experimental Data}
\label{app:extended-exp}

[Additional tables and figures for journal version]
\bibliographystyle{IEEEtran}

\bibliography{refererences_v2}

\end{document}